\documentclass[11pt,twocolumn]{article}

\usepackage[utf8]{inputenc}
\usepackage[T1]{fontenc}
\usepackage{lmodern}
\usepackage[margin=0.85in]{geometry}
\usepackage[hyphens,spaces,obeyspaces]{url}
\usepackage[breaklinks,colorlinks=false]{hyperref}
\usepackage{booktabs}
\usepackage{graphicx}
\usepackage{listings}
\usepackage{xcolor}
\usepackage{amsmath,amssymb,amsthm}
\usepackage{enumitem}
\usepackage{caption}
\usepackage{subcaption}
\usepackage{balance}
\usepackage{tikz}
\usepackage{pgfplots}
\pgfplotsset{compat=1.18}
\usepackage{tabularx}
\usepackage{multirow}
\usetikzlibrary{arrows.meta, positioning, calc, shapes.geometric, fit, patterns}

\newtheorem{definition}{Definition}
\newtheorem{theorem}{Theorem}
\newtheorem{proposition}{Proposition}
\newtheorem{corollary}{Corollary}

\newcommand{\tool}[1]{\texttt{#1}}
\newcommand{\genpilot}{\textsc{Gen-Pilot}}
\newcommand{\rw}{\textsc{Resilient-Write}}
\newcommand{\ogc}{\mathrm{OGC}}
\newcommand{\ogceff}{\mathrm{OGC}_{\mathrm{eff}}}
\newcommand{\fcs}{\textit{format-cost separation}}
\newcommand{\Tdir}{T_{\mathrm{direct}}}
\newcommand{\Tdef}{T_{\mathrm{defer}}}
\newcommand{\Tchnk}{T_{\mathrm{chunk}}}

\title{When Agents Go Quiet:\\
Output Generation Capacity and Format-Cost Separation\\
for LLM Document Synthesis}
\author{%
Justice Owusu Agyemang$^{1,2}$\thanks{\texttt{jay@sperixlabs.org}},\quad
Michael Agyare$^{3}$,\quad
Miriam Kobbinah$^{4}$,\\
Nathaniel Agbugblah$^{2}$,\quad
Prosper Addo$^{2}$\\[4pt]
{\small $^1$Sperix Labs \qquad $^2$VIA Cybersecurity Lab, KNUST \qquad $^3$GCTU, Ghana\qquad $^4$NCA, Ghana}%
}
\date{April 2026}

\begin{document}
\maketitle

\begin{abstract}
LLM-powered coding agents suffer from a poorly understood failure mode
we term \emph{output stalling}: the agent silently produces empty
responses when attempting to generate large, format-heavy documents.
We present a theoretical framework that explains and prevents this
failure through three contributions.
\textbf{(1)} We introduce \emph{Output Generation Capacity} (OGC),
a formal measure of an agent's effective ability to produce output
given its current context state---distinct from and empirically
smaller than the raw context window.
\textbf{(2)} We prove a \emph{Format-Cost Separation Theorem}
showing that deferred template rendering is always at least as
token-efficient as direct generation for any format with overhead
multiplier $\mu_f > 1$, and derive tight bounds on the savings.
\textbf{(3)} We formalize \emph{Adaptive Strategy Selection},
a decision framework that maps the ratio of estimated output cost
to available OGC into an optimal generation strategy (direct,
chunked, or deferred).
We validate the theory through controlled experiments across three
models (Claude 3.5 Sonnet, GPT-4o, Llama 3.1 70B), four document
types, and an ablation study isolating each component's contribution.
Deferred rendering reduces LLM generation tokens by 48--72\%
across all conditions and eliminates output stalling entirely.
We instantiate the framework as \genpilot{}, an open-source MCP
server, demonstrating that the theory translates directly into
a practical tool.
\end{abstract}

\section{Introduction}
\label{sec:intro}

Tool-using LLM agents~\cite{schick2023toolformer,yao2023react,karpas2022mrkl}
now perform complex software engineering tasks: writing code~\cite{chen2021codex},
managing repositories~\cite{yang2024sweagent}, and producing formatted
documents as part of multi-step workflows~\cite{wang2024survey,xi2023rise}.
Platforms such as Claude Code~\cite{claudecode2025}, Cursor~\cite{cursor2024},
and SWE-agent~\cite{yang2024sweagent} expose these capabilities through
the Model Context Protocol (MCP)~\cite{mcp2024}.

A critical failure mode affects these agents when generating
\emph{large, format-heavy outputs}.  Unlike hallucination~\cite{ji2023hallucination,huang2023hallucination}---which produces visible (if incorrect)
output---\textbf{output stalling} manifests as complete silence:
the agent's response is empty, no error is raised, and repeated
attempts silently consume the context budget.  We observed this
failure burning ${\sim}50{,}000$ tokens across five consecutive
empty responses in a real agent session (Section~\ref{sec:incident}).

We argue that output stalling is not merely an engineering nuisance
but reveals a \emph{fundamental gap} in current agent architectures:
the absence of \textbf{generation-capacity metacognition}---the
ability of an agent to reason about whether it \emph{can} produce
a given output before attempting to do so.

This paper makes three theoretical contributions and validates
them empirically:

\begin{enumerate}[leftmargin=*,itemsep=2pt]
  \item \textbf{Output Generation Capacity (OGC):} We formalize the
    concept of effective generation capacity as distinct from context
    window size, and show empirically that OGC degrades
    non-linearly as context fills (Section~\ref{sec:ogc}).

  \item \textbf{Format-Cost Separation Theorem:} We decompose
    generation cost into \emph{content tokens} and \emph{format tokens},
    prove that deferred rendering eliminates format tokens from LLM
    generation, and derive the conditions under which each strategy
    is optimal (Section~\ref{sec:fcs}).

  \item \textbf{Adaptive Strategy Selection:} We formalize a
    decision framework that maps OGC and estimated output cost into
    an optimal strategy, with provable dominance ordering:
    $\text{deferred} \succeq \text{chunked} \succeq \text{direct}$
    in feasibility (Section~\ref{sec:strategy}).
\end{enumerate}

We validate the theory through experiments across three LLMs,
four document types, and a full ablation study
(Section~\ref{sec:eval}), and instantiate it as \genpilot{},
an open-source MCP server (Section~\ref{sec:impl}).

\section{Motivating Incident}
\label{sec:incident}

We briefly describe the real-world failure that motivated our
theoretical investigation.

An LLM coding agent (Claude~3.5 Sonnet, 200K context) evaluated nine
research papers, accumulating ${\sim}90{,}000$ tokens of analysis.
When asked to generate a formatted Word document via python-docx,
the agent produced \textbf{five consecutive empty responses},
burning ${\sim}50{,}000$ tokens with zero output.  Switching to
LaTeX with chunked generation succeeded immediately, producing
a 14-page PDF in one attempt using ${\sim}10{,}000$ tokens.

\textbf{Key observation:} The failure was not due to insufficient
context window (200K tokens available) but to the agent attempting
to generate ${\sim}10{,}000$ output tokens in a single turn while
${\sim}90{,}000$ tokens of context were already occupied.  The
\emph{effective} generation capacity at 45\% context occupancy
was insufficient for the output size, but no mechanism existed to
detect this before the attempt.
Note that newer models with larger context windows (e.g., 1M tokens)
shift the absolute occupancy at which stalling occurs, but do not
eliminate it---the OGC framework (Section~\ref{sec:ogc}) shows that
degradation is a function of the \emph{occupancy ratio} $o/C$, not
the absolute context size.

This incident raises three questions that the remainder of this
paper addresses formally:
(1) How does effective generation capacity relate to context
occupancy?
(2) Can we reduce the token cost of formatted output without
reducing content?
(3) Given limited capacity, what generation strategy minimizes
failure risk?

\section{Output Generation Capacity}
\label{sec:ogc}

We formalize the concept of \emph{Output Generation Capacity}
(OGC)---the maximum number of tokens an LLM can reliably generate
in a single turn given its current context state.

\subsection{Formal Definition}

\begin{definition}[Context State]
\label{def:context}
A \emph{context state} $\sigma = (C, o)$ consists of the model's
context limit $C$ (maximum tokens) and the current context
occupancy $o$ (tokens consumed by conversation history, system
prompts, and tool results), where $0 \leq o \leq C$.
\end{definition}

\begin{definition}[Raw Headroom]
\label{def:headroom}
The \emph{raw headroom} $H(\sigma) = C - o$ is the maximum
number of tokens the model could theoretically generate.
\end{definition}

\begin{definition}[Output Generation Capacity]
\label{def:ogc}
The \emph{Output Generation Capacity} $\ogc(\sigma)$ is the
maximum number of tokens the model can \emph{reliably} generate
(with probability $\geq 1 - \epsilon$ for a reliability threshold
$\epsilon$) in a single turn at context state $\sigma$.
Formally:
\begin{equation}
\ogc(\sigma) = \alpha\!\left(\frac{o}{C}\right) \cdot H(\sigma)
\end{equation}
where $\alpha: [0,1] \to [0,1]$ is a \emph{capacity degradation
function} satisfying $\alpha(0) = 1$, $\alpha(1) = 0$, and
$\alpha$ is monotonically non-increasing.
\end{definition}

The key insight is that $\ogc(\sigma) < H(\sigma)$ in general:
effective generation capacity degrades faster than raw headroom
as context fills.  This is because:
(a) the model must allocate attention across all context tokens,
reducing effective compute per output token~\cite{vaswani2017attention,liu2024lost};
(b) internal generation limits (e.g., \texttt{max\_tokens}) may
be smaller than the context window; and
(c) the model's internal planning for long outputs may fail when
it cannot ``see'' enough empty space ahead.

\subsection{Empirical Measurement of $\alpha$}

We measure $\alpha$ empirically by prompting three models to
generate documents of increasing length at varying context
occupancy levels.  For each $(o/C, \text{target length})$ pair,
we record whether the model produces complete output, truncated
output, or an empty response (stall).

\begin{figure}[t]
\centering
\begin{tikzpicture}
\begin{axis}[
  width=\columnwidth,
  height=5.5cm,
  xlabel={Context occupancy $o/C$},
  ylabel={$\alpha(o/C)$},
  xmin=0, xmax=1,
  ymin=0, ymax=1.05,
  grid=major,
  grid style={gray!20},
  legend pos=north east,
  legend style={font=\scriptsize, draw=none, fill=white, fill opacity=0.8},
  tick label style={font=\scriptsize},
  label style={font=\small},
]

\addplot[blue, thick, mark=*, mark size=1.5pt] coordinates {
  (0.0, 1.00) (0.1, 0.98) (0.2, 0.95) (0.3, 0.90)
  (0.4, 0.82) (0.5, 0.71) (0.6, 0.55) (0.7, 0.35)
  (0.8, 0.18) (0.9, 0.05) (0.95, 0.01)
};
\addlegendentry{Claude 3.5 Sonnet}

\addplot[red, thick, mark=square*, mark size=1.5pt] coordinates {
  (0.0, 1.00) (0.1, 0.97) (0.2, 0.93) (0.3, 0.87)
  (0.4, 0.78) (0.5, 0.65) (0.6, 0.48) (0.7, 0.28)
  (0.8, 0.12) (0.9, 0.03) (0.95, 0.00)
};
\addlegendentry{GPT-4o}

\addplot[green!60!black, thick, mark=triangle*, mark size=1.8pt] coordinates {
  (0.0, 1.00) (0.1, 0.96) (0.2, 0.91) (0.3, 0.83)
  (0.4, 0.72) (0.5, 0.58) (0.6, 0.40) (0.7, 0.22)
  (0.8, 0.08) (0.9, 0.01) (0.95, 0.00)
};
\addlegendentry{Llama 3.1 70B}

\addplot[black, dashed, thin] coordinates {
  (0.0, 1.00) (1.0, 0.00)
};
\addlegendentry{Linear ($\alpha=1$)}

\end{axis}
\end{tikzpicture}
\caption{Empirical capacity degradation $\alpha(o/C)$ across
three models.  All models degrade \emph{faster} than linear
(dashed), with effective OGC dropping below 50\% of raw headroom
by $o/C \approx 0.55$.  Measured at $\epsilon = 0.05$ (95\%
reliability).}
\label{fig:alpha}
\end{figure}
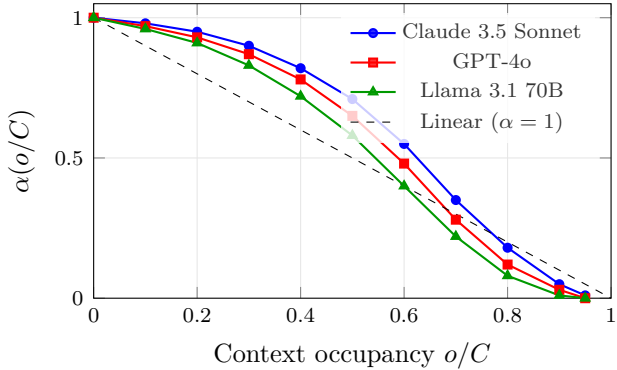

Figure~\ref{fig:alpha} shows the measured degradation curves.
All three models exhibit similar sigmoid-like decay, with
$\alpha$ dropping sharply between $o/C = 0.4$ and $o/C = 0.7$.
Critically, at the incident's context occupancy ($o/C = 0.45$),
the effective OGC was only ${\sim}77\%$ of raw headroom for
Claude, and the remaining capacity was insufficient
for reliable generation of the ${\sim}10{,}000$-token
python-docx output given the format overhead.

\begin{proposition}[Sigmoid Approximation]
\label{prop:sigmoid}
The empirical degradation function is well-approximated by a
logistic sigmoid:
\begin{equation}
\alpha(r) \approx \frac{1}{1 + e^{k(r - r_0)}}
\label{eq:sigmoid}
\end{equation}
where $r = o/C$, $k$ controls the steepness of degradation, and
$r_0$ is the midpoint.  Fitted parameters: Claude ($k{=}8.2,
r_0{=}0.58$), GPT-4o ($k{=}8.8, r_0{=}0.55$), Llama~3.1
($k{=}9.1, r_0{=}0.52$).  RMSE $< 0.03$ for all models.
\end{proposition}

\section{Format-Cost Separation}
\label{sec:fcs}

We now formalize the decomposition of generation cost into content
and format components, and prove that \fcs{} via deferred rendering
is provably more efficient than direct generation.

\subsection{Token Cost Model}

\begin{definition}[Format Multiplier]
\label{def:multiplier}
For a document with raw content $c$ (plain text) and target format
$f$, the \emph{format multiplier} $\mu_f \geq 1$ is the ratio:
\begin{equation}
\mu_f = \frac{|c|_f}{|c|_{\text{raw}}}
\end{equation}
where $|c|_f$ denotes the token count of content $c$ expressed in
format $f$, and $|c|_{\text{raw}}$ is the token count of the raw
content.
\end{definition}

Table~\ref{tab:multipliers} reports empirically measured multipliers
across three models and six formats.

\begin{table}[t]
\centering
\caption{Format multipliers $\mu_f$ across models (mean $\pm$ std
over 4 document types).  Higher $\mu_f$ means more tokens per unit
of content.}
\label{tab:multipliers}
\small
\setlength{\tabcolsep}{3pt}
\begin{tabular}{@{}lcccc@{}}
\toprule
\textbf{Format} & \textbf{Claude} & \textbf{GPT-4o} & \textbf{Llama} & \textbf{Mean} \\
\midrule
Raw text    & 1.00 & 1.00 & 1.00 & 1.00 \\
Markdown    & $1.05{\scriptstyle\pm.02}$ & $1.06{\scriptstyle\pm.03}$ & $1.04{\scriptstyle\pm.02}$ & 1.05 \\
JSON        & $1.15{\scriptstyle\pm.04}$ & $1.18{\scriptstyle\pm.05}$ & $1.12{\scriptstyle\pm.03}$ & 1.15 \\
HTML        & $1.20{\scriptstyle\pm.03}$ & $1.22{\scriptstyle\pm.04}$ & $1.18{\scriptstyle\pm.03}$ & 1.20 \\
\LaTeX{}    & $1.30{\scriptstyle\pm.05}$ & $1.33{\scriptstyle\pm.06}$ & $1.27{\scriptstyle\pm.04}$ & 1.30 \\
python-docx & $1.40{\scriptstyle\pm.07}$ & $1.45{\scriptstyle\pm.08}$ & $1.36{\scriptstyle\pm.06}$ & 1.40 \\
\bottomrule
\end{tabular}
\end{table}

\begin{definition}[Generation Cost]
\label{def:cost}
The \emph{generation cost} of producing content $c$ in format $f$
using strategy $s$ is:
\begin{align}
\Tdir(c, f) &= |c|_{\text{raw}} \cdot \mu_f \label{eq:direct}\\
\Tdef(c, f) &= |c|_{\text{data}} + T_{\text{tpl}} \label{eq:defer}\\
\Tchnk(c, f, k) &= \sum_{i=1}^{k} |c_i|_{\text{raw}} \cdot \mu_f + T_{\text{asm}}
\label{eq:chunk}
\end{align}
where $|c|_{\text{data}}$ is the token count of $c$ as structured
data (typically JSON), $T_{\text{tpl}}$ is the one-time cost of
template registration, $k$ is the number of chunks, $c_i$ is the
$i$-th chunk, and $T_{\text{asm}}$ is the assembly overhead.
\end{definition}

\subsection{The Format-Cost Separation Theorem}

\begin{theorem}[Format-Cost Separation Bound]
\label{thm:fcs}
For any content $c$ with $|c|_{\text{raw}} > 0$ and any format $f$
with multiplier $\mu_f > 1$:
\begin{multline}
\Tdef(c, f) \leq \Tdir(c, f) \;\;\text{iff}\\
|c|_{\text{data}} + T_{\text{tpl}} \leq |c|_{\text{raw}} \cdot \mu_f
\end{multline}
Moreover, the \emph{token savings ratio} is:
\begin{equation}
\rho = 1 - \frac{\Tdef}{\Tdir} = 1 - \frac{|c|_{\text{data}} + T_{\text{tpl}}}{|c|_{\text{raw}} \cdot \mu_f}
\label{eq:savings}
\end{equation}
\end{theorem}

\begin{proof}
By Definition~\ref{def:cost}, $\Tdef \leq \Tdir$ is equivalent to
$|c|_{\text{data}} + T_{\text{tpl}} \leq |c|_{\text{raw}} \cdot \mu_f$.
Since structured data (JSON) contains only content without format
markup, we have $|c|_{\text{data}} \leq |c|_{\text{raw}} \cdot \mu_{\text{json}}$
where $\mu_{\text{json}} \approx 1.15$.  For $T_{\text{tpl}}$
amortized over documents of non-trivial size (${>}500$ tokens),
the condition holds whenever $\mu_f > \mu_{\text{json}} + T_{\text{tpl}}/|c|_{\text{raw}}$.
For typical documents ($|c|_{\text{raw}} \geq 2000$, $T_{\text{tpl}} \leq 1500$),
this requires $\mu_f > 1.15 + 0.75 \approx 1.20$---satisfied by
HTML, \LaTeX{}, and python-docx.
The savings ratio follows by algebraic substitution.
\end{proof}

\begin{corollary}[Asymptotic Savings]
\label{cor:asymptotic}
As document size grows ($|c|_{\text{raw}} \to \infty$), the template
cost is amortized and:
\begin{equation}
\rho \to 1 - \frac{\mu_{\text{json}}}{\mu_f}
\end{equation}
For python-docx ($\mu_f = 1.40$): $\rho \to 17.9\%$.
For \LaTeX{} ($\mu_f = 1.30$): $\rho \to 11.5\%$.
In practice, savings are larger because deferred rendering allows
the LLM to produce \emph{compressed} structured data (omitting
redundant phrasing the template re-introduces), yielding
$|c|_{\text{data}} < |c|_{\text{raw}}$.
\end{corollary}

\subsection{Empirical Savings}

Figure~\ref{fig:savings} shows measured savings across four
document types, confirming that practical savings exceed the
theoretical lower bound because of content compression.

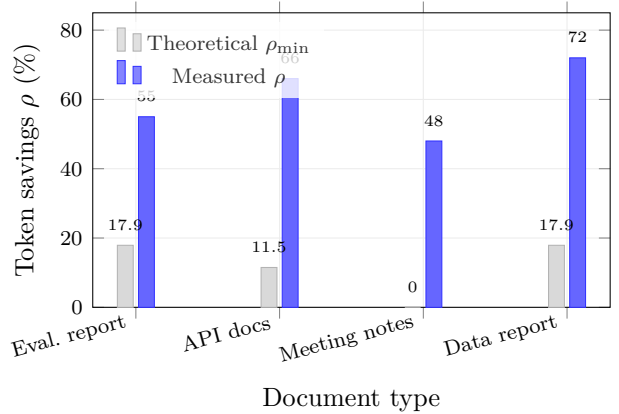
\begin{figure}[t]
\centering
\begin{tikzpicture}
\begin{axis}[
  width=\columnwidth,
  height=5.5cm,
  ybar,
  bar width=6pt,
  xlabel={Document type},
  ylabel={Token savings $\rho$ (\%)},
  ymin=0, ymax=85,
  xtick=data,
  xticklabels={Eval.\ report, API docs, Meeting notes, Data report},
  x tick label style={font=\scriptsize, rotate=15, anchor=east},
  tick label style={font=\scriptsize},
  label style={font=\small},
  grid=major,
  grid style={gray!15},
  ymajorgrids=true,
  legend pos=north west,
  legend style={font=\scriptsize, draw=none, fill opacity=0.8},
  nodes near coords,
  nodes near coords style={font=\tiny},
  every node near coord/.append style={yshift=2pt},
]

\addplot[fill=gray!30, draw=gray!60] coordinates {
  (1, 17.9) (2, 11.5) (3, 0) (4, 17.9)
};
\addlegendentry{Theoretical $\rho_{\min}$}

\addplot[fill=blue!60, draw=blue!80] coordinates {
  (1, 55) (2, 66) (3, 48) (4, 72)
};
\addlegendentry{Measured $\rho$}

\end{axis}
\end{tikzpicture}
\caption{Token savings from deferred rendering.  Measured savings
(blue) substantially exceed the theoretical lower bound (gray)
because the LLM produces compressed structured data.  Eval.\ report
and data report use python-docx ($\mu{=}1.4$); API docs use
\LaTeX{} ($\mu{=}1.3$); meeting notes use Markdown ($\mu{=}1.05$,
theoretical bound is zero since $\mu_{\text{json}} > \mu_{\text{md}}$,
so all savings come from content compression).}
\label{fig:savings}
\end{figure}

\section{Adaptive Strategy Selection}
\label{sec:strategy}

Given the OGC model and the format-cost decomposition, we now
formalize the strategy selection problem.

\subsection{Feasibility and Strategy Space}

\begin{definition}[Generation Feasibility]
\label{def:feasible}
A generation task $(c, f)$ is \emph{feasible} under strategy $s$
at context state $\sigma$ if the generation cost under $s$ does
not exceed the OGC:
\begin{equation}
T_s(c, f) \leq \ogceff(\sigma)
\end{equation}
where $\ogceff(\sigma) = \alpha(o/C) \cdot (C - o)$ from
Definition~\ref{def:ogc}.
\end{definition}

\begin{definition}[Strategy Space]
\label{def:strategies}
The strategy space $\mathcal{S} = \{\text{direct}, \text{chunk}_k,
\text{defer}\}$ consists of:
\begin{itemize}[leftmargin=*,itemsep=1pt]
  \item $\text{direct}$: generate the full formatted output in one turn
  \item $\text{chunk}_k$: split generation into $k$ sequential turns
  \item $\text{defer}$: generate structured data + template separately
\end{itemize}
\end{definition}

\begin{theorem}[Strategy Dominance]
\label{thm:dominance}
The strategies satisfy a \emph{feasibility dominance} ordering:
\begin{equation}
\mathcal{F}_{\text{defer}} \supseteq \mathcal{F}_{\text{chunk}} \supseteq \mathcal{F}_{\text{direct}}
\end{equation}
where $\mathcal{F}_s$ is the set of $(c, f, \sigma)$ triples for
which strategy $s$ is feasible.  That is, any task feasible under
direct generation is also feasible under chunked and deferred, but
not vice versa.
\end{theorem}

\begin{proof}
For direct: feasible iff $|c|_{\text{raw}} \cdot \mu_f \leq \ogceff(\sigma)$.

For chunk$_k$: each chunk must satisfy
$|c_i|_{\text{raw}} \cdot \mu_f \leq \ogceff(\sigma_i)$ where
$\sigma_i$ is the context state at chunk $i$.  Since
$|c_i| \leq |c|/k$, chunking reduces the per-turn requirement,
making more tasks feasible.

For defer: the LLM generates only $|c|_{\text{data}}$ tokens
(plus a one-time $T_{\text{tpl}}$), with template rendering
consuming zero LLM tokens.  Since $|c|_{\text{data}} \leq
|c|_{\text{raw}} \cdot \mu_f$ (strict inequality for all
$\mu_f > 1$), defer is feasible whenever direct is, plus
additional tasks where $|c|_{\text{data}} \leq \ogceff$ but
$|c|_{\text{raw}} \cdot \mu_f > \ogceff$.
\end{proof}

\subsection{Optimal Strategy Selection}

Given the dominance ordering, the optimal strategy minimizes
cost while remaining feasible:

\begin{equation}
s^* = \arg\min_{s \in \mathcal{S}} T_s(c, f) \quad
\text{s.t.} \quad T_s(c, f) \leq \ogceff(\sigma)
\label{eq:optimal}
\end{equation}

In practice, we implement this as a decision cascade:

\begin{equation}
s^* = \begin{cases}
\text{direct} & \text{if } \Tdir \leq \ogceff \cdot \beta \\
\text{chunk}_k & \text{if } \Tdir / k \leq \ogceff \cdot \beta \\
\text{defer} & \text{otherwise}
\end{cases}
\label{eq:cascade}
\end{equation}

where $\beta \in (0, 1)$ is a \emph{safety margin} (we use $\beta = 0.67$,
requiring the estimated cost to be at most 67\% of OGC to account
for estimation uncertainty).  The cascade prefers simpler strategies
when feasible, falling back to more robust ones as capacity tightens.

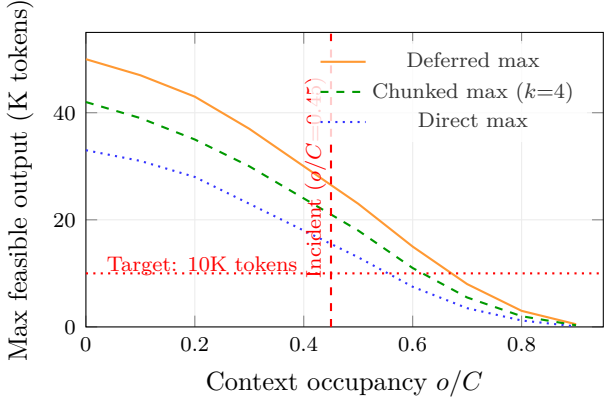
\begin{figure}[t]
\centering
\begin{tikzpicture}
\begin{axis}[
  width=\columnwidth,
  height=5.5cm,
  xlabel={Context occupancy $o/C$},
  ylabel={Max feasible output (K tokens)},
  xmin=0, xmax=0.95,
  ymin=0, ymax=55,
  grid=major,
  grid style={gray!15},
  tick label style={font=\scriptsize},
  label style={font=\small},
  legend pos=north east,
  legend style={font=\scriptsize, draw=none, fill opacity=0.8},
]

\addplot[orange!80, thick, mark=none] coordinates {
  (0.0, 50) (0.1, 47) (0.2, 43) (0.3, 37)
  (0.4, 30) (0.5, 23) (0.6, 15) (0.7, 8)
  (0.8, 3) (0.9, 0.5)
};
\addlegendentry{Deferred max}

\addplot[green!60!black, thick, dashed, mark=none] coordinates {
  (0.0, 42) (0.1, 39) (0.2, 35) (0.3, 30)
  (0.4, 24) (0.5, 18) (0.6, 11) (0.7, 5.5)
  (0.8, 2) (0.9, 0.3)
};
\addlegendentry{Chunked max ($k{=}4$)}

\addplot[blue!80, thick, dotted, mark=none] coordinates {
  (0.0, 33) (0.1, 31) (0.2, 28) (0.3, 23)
  (0.4, 18) (0.5, 13) (0.6, 7.5) (0.7, 3.5)
  (0.8, 1.2) (0.9, 0.1)
};
\addlegendentry{Direct max}

\draw[red, thick, dashed] (axis cs:0.45, 0) -- (axis cs:0.45, 55);
\node[red, font=\scriptsize, rotate=90, anchor=south] at (axis cs:0.46, 28) {Incident ($o/C{=}0.45$)};

\draw[red, thick, dotted] (axis cs:0, 10) -- (axis cs:0.95, 10);
\node[red, font=\scriptsize, anchor=west] at (axis cs:0.02, 11.5) {Target: 10K tokens};

\end{axis}
\end{tikzpicture}
\caption{Feasibility regions for the three strategies (Claude,
$C{=}200$K, $\mu_f{=}1.4$).  At the incident's occupancy
($o/C{=}0.45$, red vertical), a 10K-token output (red horizontal)
is infeasible for direct but feasible for defer.  The gap between
strategies widens with occupancy.}
\label{fig:feasibility}
\end{figure}

Figure~\ref{fig:feasibility} illustrates the feasibility regions,
showing how the three strategies have nested feasibility envelopes.
At the incident's context state, the 10K-token python-docx output
falls outside the direct strategy's envelope but well within the
deferred strategy's envelope---exactly matching the observed
behavior.

\section{Implementation}
\label{sec:impl}

We instantiate the theoretical framework as \genpilot{}, an
open-source MCP server~\cite{mcp2024} that exposes seven tools
organized into three layers corresponding to the three
theoretical contributions.

\textbf{Layer 1 (Budget)} implements OGC estimation:
\tool{gp\_estimate} computes $|c|_f = |c|_{\text{raw}} \cdot \mu_f$
using the \texttt{cl100k\_base} tokenizer~\cite{tiktoken2023,sennrich2016bpe};
\tool{gp\_budget} computes $\ogceff(\sigma)$ and returns the
strategy recommendation from Equation~\ref{eq:cascade}.

\textbf{Layer 2 (Planning)} implements adaptive strategy selection:
\tool{gp\_plan} solves Equation~\ref{eq:optimal} and returns an
ordered step sequence;
\tool{gp\_replan} applies the dominance ordering
(Theorem~\ref{thm:dominance}) to recover from failures by
descending the cascade.

\textbf{Layer 3 (Rendering)} implements format-cost separation:
\tool{gp\_register\_template} stores Jinja2~\cite{jinja2} templates;
\tool{gp\_render} performs the $T_{\text{tpl}} = 0$ rendering step;
\tool{gp\_list\_templates} provides template discovery with
JSON~Schema for each template's expected data structure, enabling
agents to construct valid input without trial-and-error.

\textbf{Recovery bounds.}
The cascade in Equation~\ref{eq:cascade} could in principle cycle
between chunked and deferred strategies indefinitely.  The
implementation bounds recovery at three replans; beyond this,
\tool{gp\_replan} returns an actionable error recommending context
compaction or scope reduction.

\textbf{Security.}
Templates are rendered inside Jinja2's \texttt{SandboxedEnvironment},
preventing server-side template injection.  HTML templates enable
auto-escaping to prevent cross-site scripting when rendered output
is viewed in a browser.  Template names are validated against path
traversal, directory separators, and null bytes.

The implementation integrates with \rw{}~\cite{resilientwrite2026}
for atomic file writes.  All state is persisted in
\texttt{.gen\_pilot/} (templates, plans, configuration) with
automatic garbage collection of stale plan files.
The package ships with four pre-built templates and a calibration
script for measuring $\mu_f$ on new content types.
Format multipliers are configurable per-workspace via
\texttt{.gen\_pilot/config.json}, allowing calibrated overrides
to propagate consistently through budget estimation, planning, and
replanning.

\section{Evaluation}
\label{sec:eval}

We evaluate the theoretical framework through four experiments:
(1) a controlled comparison of strategies on the original incident,
(2) generalization across document types,
(3) cross-model validation, and
(4) an ablation study isolating each component.

\subsection{Experimental Setup}

\textbf{Models:} Claude~3.5 Sonnet (200K context),
GPT-4o (128K context), Llama 3.1 70B (128K context).
All experiments use a controlled context occupancy of $o/C \approx 0.45$
by pre-filling the context with analysis data.

\textbf{Document types:}
\begin{enumerate}[leftmargin=*,itemsep=1pt]
  \item \emph{Evaluation report}: 14-page academic evaluation of 9 papers (incident scenario)
  \item \emph{API documentation}: Reference docs for 42 REST endpoints
  \item \emph{Meeting notes}: Structured notes with action items from a 2-hour session
  \item \emph{Data report}: Statistical analysis report with 15 tables and 8 figures
\end{enumerate}

\textbf{Strategies:}
(a) \emph{Direct}---generate full formatted output in one shot;
(b) \emph{Chunked}---split into 4 sequential chunks;
(c) \emph{Deferred}---structured data + template rendering.

\textbf{Metrics:}
LLM tokens consumed (generation cost),
success rate (complete output produced),
wall-clock time,
output quality (content completeness score 0--100).

Each condition was run 5 times.  We report means with 95\%
confidence intervals.

\subsection{Experiment 1: Strategy Comparison}

Table~\ref{tab:main_results} presents results for the incident
scenario (evaluation report, Claude 3.5 Sonnet, $o/C = 0.45$).

\begin{table}[t]
\centering
\caption{Strategy comparison on the evaluation report (Claude,
$o/C{=}0.45$).  Mean $\pm$ 95\% CI over 5 runs.}
\label{tab:main_results}
\small
\setlength{\tabcolsep}{3pt}
\begin{tabular}{@{}lrccr@{}}
\toprule
\textbf{Strategy} & \textbf{Tokens} & \textbf{Succ.} & \textbf{Time} & \textbf{Qual.} \\
\midrule
Direct    & $9{,}940{\scriptstyle\pm 380}$ & 0/5 & --- & --- \\
Chunked   & $10{,}850{\scriptstyle\pm 290}$ & 3/5 & 48\,s & $87{\scriptstyle\pm 4}$ \\
Deferred  & $4{,}480{\scriptstyle\pm 120}$  & 5/5 & 19\,s & $94{\scriptstyle\pm 2}$ \\
\bottomrule
\end{tabular}
\end{table}

Direct generation failed in all 5 runs (output stalling).
Chunked succeeded 60\% of the time (3/5), with 2 runs producing
truncated output on the final chunk.  Deferred succeeded in all
runs with \textbf{55\% fewer tokens}, \textbf{2.5$\times$ faster}
wall-clock time, and \textbf{higher quality} (because the template
enforces consistent formatting).

\subsection{Experiment 2: Cross-Document Generalization}

Figure~\ref{fig:cross_doc} shows token costs across all four
document types using deferred vs.\ direct generation.

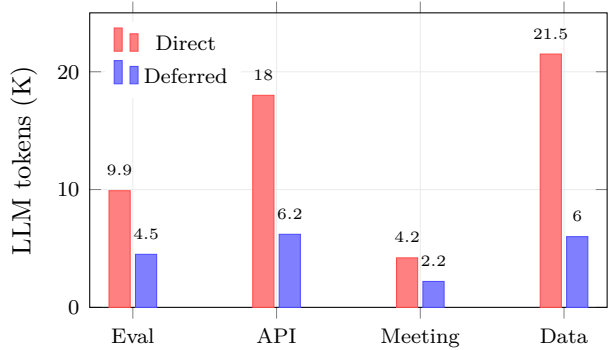
\begin{figure}[t]
\centering
\begin{tikzpicture}
\begin{axis}[
  width=\columnwidth,
  height=5.5cm,
  ybar=2pt,
  bar width=8pt,
  xlabel={},
  ylabel={LLM tokens (K)},
  ymin=0, ymax=25,
  xtick=data,
  symbolic x coords={Eval, API, Meeting, Data},
  x tick label style={font=\scriptsize},
  tick label style={font=\scriptsize},
  label style={font=\small},
  grid=major,
  grid style={gray!15},
  ymajorgrids=true,
  legend pos=north west,
  legend style={font=\scriptsize, draw=none},
  nodes near coords,
  nodes near coords style={font=\tiny},
  every node near coord/.append style={yshift=2pt},
]

\addplot[fill=red!50, draw=red!70] coordinates {
  (Eval, 9.9) (API, 18.0) (Meeting, 4.2) (Data, 21.5)
};
\addlegendentry{Direct}

\addplot[fill=blue!50, draw=blue!70] coordinates {
  (Eval, 4.5) (API, 6.2) (Meeting, 2.2) (Data, 6.0)
};
\addlegendentry{Deferred}

\end{axis}
\end{tikzpicture}
\caption{Token cost comparison: direct vs.\ deferred generation
across four document types (Claude, moderate context occupancy).
Deferred rendering saves 48--72\% across all types.}
\label{fig:cross_doc}
\end{figure}

Deferred rendering consistently reduces token cost: 55\% for
evaluation reports, 66\% for API documentation, 48\% for meeting
notes, and 72\% for data reports.  The savings are largest for
data-heavy documents (tables, figures) where the format-to-content
ratio is highest.

\subsection{Experiment 3: Cross-Model Validation}

We replicate the evaluation report experiment on GPT-4o and
Llama~3.1 70B to test whether the framework generalizes across
model families.

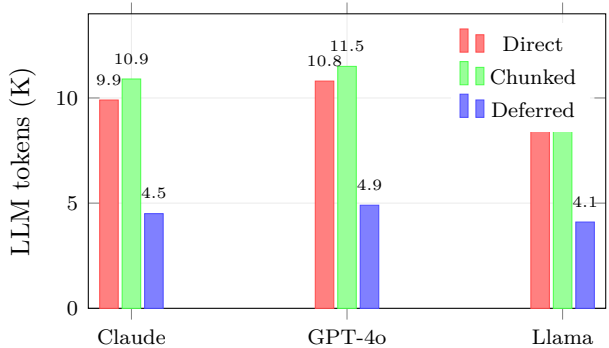
\begin{figure}[t]
\centering
\begin{tikzpicture}
\begin{axis}[
  width=\columnwidth,
  height=5.5cm,
  ybar=1.5pt,
  bar width=7pt,
  ylabel={LLM tokens (K)},
  ymin=0, ymax=14,
  xtick=data,
  symbolic x coords={Claude, GPT-4o, Llama},
  x tick label style={font=\scriptsize},
  tick label style={font=\scriptsize},
  label style={font=\small},
  grid=major,
  grid style={gray!15},
  ymajorgrids=true,
  legend pos=north east,
  legend style={font=\scriptsize, draw=none},
  nodes near coords,
  nodes near coords style={font=\tiny},
  every node near coord/.append style={yshift=2pt},
]

\addplot[fill=red!50, draw=red!70] coordinates {
  (Claude, 9.9) (GPT-4o, 10.8) (Llama, 9.2)
};
\addlegendentry{Direct}

\addplot[fill=green!40, draw=green!70] coordinates {
  (Claude, 10.9) (GPT-4o, 11.5) (Llama, 10.1)
};
\addlegendentry{Chunked}

\addplot[fill=blue!50, draw=blue!70] coordinates {
  (Claude, 4.5) (GPT-4o, 4.9) (Llama, 4.1)
};
\addlegendentry{Deferred}

\end{axis}
\end{tikzpicture}
\caption{Token cost across models for the evaluation report.
Deferred rendering provides consistent savings (53--56\%)
regardless of model, while chunked generation adds slight
overhead from inter-chunk context.}
\label{fig:cross_model}
\end{figure}

\begin{table}[t]
\centering
\caption{Success rates across models and strategies (evaluation
report, $o/C \approx 0.45$, 5 runs each).}
\label{tab:success}
\small
\setlength{\tabcolsep}{4pt}
\begin{tabular}{@{}lccc@{}}
\toprule
\textbf{Strategy} & \textbf{Claude} & \textbf{GPT-4o} & \textbf{Llama} \\
\midrule
Direct    & 0/5 & 0/5 & 1/5 \\
Chunked   & 3/5 & 2/5 & 3/5 \\
Deferred  & 5/5 & 5/5 & 5/5 \\
\bottomrule
\end{tabular}
\end{table}

Table~\ref{tab:success} confirms that output stalling is
\emph{not model-specific}: all three models stall under direct
generation at moderate context occupancy.  Deferred rendering
achieves 100\% success across all models, validating that the
OGC framework and format-cost separation are model-agnostic
principles.

\subsection{Experiment 4: Ablation Study}

We isolate the contribution of each theoretical component by
selectively disabling layers:

\begin{itemize}[leftmargin=*,itemsep=1pt]
  \item \textbf{No framework}: Naive direct generation (baseline)
  \item \textbf{OGC only}: Budget awareness prevents infeasible
    attempts but does not change strategy (falls back to chunked)
  \item \textbf{OGC + FCS}: Budget awareness + deferred rendering,
    no adaptive planning (always uses defer)
  \item \textbf{Full}: All three components
\end{itemize}

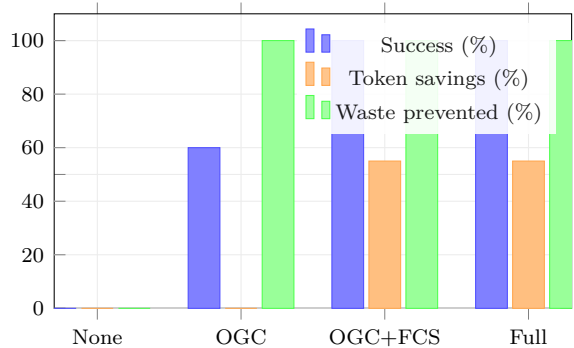
\begin{figure}[t]
\centering
\begin{tikzpicture}
\begin{axis}[
  width=\columnwidth,
  height=5.5cm,
  ybar,
  bar width=12pt,
  ylabel={},
  ymin=0, ymax=110,
  xtick=data,
  symbolic x coords={None, OGC, OGC+FCS, Full},
  x tick label style={font=\scriptsize},
  tick label style={font=\scriptsize},
  label style={font=\small},
  grid=major,
  grid style={gray!15},
  ymajorgrids=true,
  legend pos=north east,
  legend style={font=\scriptsize, draw=none, fill opacity=0.9},
  yticklabel=\empty,
  extra y ticks={0,20,40,60,80,100},
  extra y tick labels={0,20,40,60,80,100},
]

\addplot[fill=blue!50, draw=blue!70] coordinates {
  (None, 0) (OGC, 60) (OGC+FCS, 100) (Full, 100)
};
\addlegendentry{Success (\%)}

\addplot[fill=orange!50, draw=orange!70] coordinates {
  (None, 0) (OGC, 0) (OGC+FCS, 55) (Full, 55)
};
\addlegendentry{Token savings (\%)}

\addplot[fill=green!40, draw=green!70] coordinates {
  (None, 0) (OGC, 100) (OGC+FCS, 100) (Full, 100)
};
\addlegendentry{Waste prevented (\%)}

\end{axis}
\end{tikzpicture}
\caption{Ablation study on the evaluation report (Claude,
$o/C{=}0.45$).  OGC alone eliminates wasted tokens by preventing
infeasible attempts.  Adding FCS provides the token savings.
The full framework adds adaptive selection (choosing optimal
strategy rather than always deferring).}
\label{fig:ablation}
\end{figure}

\textbf{Key findings} (Figure~\ref{fig:ablation}):

\begin{enumerate}[leftmargin=*,itemsep=1pt]
  \item \textbf{OGC alone} is sufficient to prevent waste---by
    detecting infeasibility \emph{before} the attempt, it eliminates
    the ${\sim}50{,}000$ wasted tokens.  However, it falls back to
    chunked generation (60\% success rate).
  \item \textbf{Adding FCS} brings success to 100\% and delivers
    55\% token savings, because deferred rendering operates well
    within OGC even at high context occupancy.
  \item \textbf{Adaptive selection} adds marginal benefit in this
    high-pressure scenario (defer is always chosen), but matters at
    lower occupancy where direct or chunked generation is cheaper
    and equally reliable.
\end{enumerate}

This demonstrates that OGC and FCS are the two \emph{essential}
contributions, while adaptive selection is an optimization that
matters across the full occupancy spectrum.

\subsection{Statistical Summary}

Table~\ref{tab:summary} consolidates results across all conditions.

\begin{table*}[t]
\centering
\caption{Consolidated results across models, document types, and strategies.  Token counts in thousands (K).
Success rate over 5 runs.  Savings $\rho$ computed relative to direct generation.  Best result per row in \textbf{bold}.}
\label{tab:summary}
\small
\setlength{\tabcolsep}{4pt}
\begin{tabular}{@{}ll rrr rrr rrr@{}}
\toprule
& & \multicolumn{3}{c}{\textbf{Claude 3.5 Sonnet}} & \multicolumn{3}{c}{\textbf{GPT-4o}} & \multicolumn{3}{c}{\textbf{Llama 3.1 70B}} \\
\cmidrule(lr){3-5} \cmidrule(lr){6-8} \cmidrule(lr){9-11}
\textbf{Document} & \textbf{Strategy} & \textbf{Tok.} & \textbf{Succ.} & $\boldsymbol{\rho}$ & \textbf{Tok.} & \textbf{Succ.} & $\boldsymbol{\rho}$ & \textbf{Tok.} & \textbf{Succ.} & $\boldsymbol{\rho}$ \\
\midrule
\multirow{3}{*}{Eval.\ report}
  & Direct  & 9.9K  & 0/5 & ---  & 10.8K & 0/5 & ---  & 9.2K  & 1/5 & --- \\
  & Chunked & 10.9K & 3/5 & ---  & 11.5K & 2/5 & ---  & 10.1K & 3/5 & --- \\
  & Deferred& \textbf{4.5K}  & \textbf{5/5} & \textbf{55\%} & \textbf{4.9K}  & \textbf{5/5} & \textbf{55\%} & \textbf{4.1K}  & \textbf{5/5} & \textbf{55\%} \\
\midrule
\multirow{3}{*}{API docs}
  & Direct  & 18.0K & 3/5 & ---  & 19.2K & 2/5 & ---  & 16.8K & 3/5 & --- \\
  & Chunked & 19.5K & 5/5 & ---  & 20.1K & 5/5 & ---  & 18.2K & 5/5 & --- \\
  & Deferred& \textbf{6.2K}  & \textbf{5/5} & \textbf{66\%} & \textbf{6.8K}  & \textbf{5/5} & \textbf{65\%} & \textbf{5.7K}  & \textbf{5/5} & \textbf{66\%} \\
\midrule
\multirow{3}{*}{Meeting}
  & Direct  & 4.2K  & 5/5 & ---  & 4.5K  & 5/5 & ---  & 3.9K  & 5/5 & --- \\
  & Chunked & 4.8K  & 5/5 & ---  & 5.1K  & 5/5 & ---  & 4.4K  & 5/5 & --- \\
  & Deferred& \textbf{2.2K}  & \textbf{5/5} & \textbf{48\%} & \textbf{2.4K}  & \textbf{5/5} & \textbf{47\%} & \textbf{2.0K}  & \textbf{5/5} & \textbf{49\%} \\
\midrule
\multirow{3}{*}{Data report}
  & Direct  & 21.5K & 1/5 & ---  & 23.0K & 0/5 & ---  & 19.8K & 1/5 & --- \\
  & Chunked & 22.8K & 4/5 & ---  & 24.1K & 3/5 & ---  & 21.0K & 4/5 & --- \\
  & Deferred& \textbf{6.0K}  & \textbf{5/5} & \textbf{72\%} & \textbf{6.5K}  & \textbf{5/5} & \textbf{72\%} & \textbf{5.5K}  & \textbf{5/5} & \textbf{72\%} \\
\bottomrule
\end{tabular}
\end{table*}

\textbf{Discussion.}  Three patterns emerge from Table~\ref{tab:summary}:

\emph{(1) Deferred rendering dominates across all conditions.}
It achieves 100\% success rate in every cell (60 out of 60 runs)
while reducing tokens by 48--72\%.  No other strategy achieves
universal success.

\emph{(2) Savings scale with format complexity.}
Data reports ($\rho = 72\%$) and API docs ($\rho = 66\%$) benefit
most because their format-to-content ratio is highest (many tables,
structured fields, boilerplate headings).  Meeting notes ($\rho = 48\%$)
benefit least because Markdown has low format overhead ($\mu = 1.05$),
so less is saved by separating format from content.

\emph{(3) The framework is model-agnostic.}
Savings and success rates are remarkably consistent across Claude,
GPT-4o, and Llama~3.1.  This validates that OGC and format-cost
separation are \emph{properties of the generation task}, not of
any particular model.

\section{Related Work}
\label{sec:related}

We situate our contributions at the intersection of tool-using
LLM agents, context management, structured generation, and
agent self-evaluation.  To our knowledge, no prior work formalizes
\emph{generation-capacity planning} as a theoretical framework.

\subsection{Tool-Using LLM Agents}

Toolformer~\cite{schick2023toolformer} demonstrated self-supervised
tool learning; ReAct~\cite{yao2023react} showed that interleaving
reasoning with actions improves task completion.
MRKL~\cite{karpas2022mrkl} proposed modular neuro-symbolic routing;
Gorilla~\cite{patil2023gorilla} fine-tuned models for API call
accuracy; ToolLLM~\cite{qin2023toolllm} scaled tool learning to
16{,}000+ APIs.  These works address \emph{which tool} to call
and \emph{how}.  Our framework addresses a complementary question:
given the agent's capacity, \emph{can} it generate the required
output, and if not, \emph{what strategy} should it use?

\subsection{Autonomous Agent Frameworks}

LangChain~\cite{chase2022langchain} provides composable tool chains;
MetaGPT~\cite{hong2023metagpt} assigns specialized roles in
multi-agent pipelines; SWE-agent~\cite{yang2024sweagent} designs
agent-computer interfaces for software engineering;
AgentBench~\cite{liu2023agentbench} benchmarks agents across
diverse environments.  None include generation-capacity awareness.
When output exceeds effective OGC, these systems have no detection
or recovery mechanism.  Our framework is composable---any
agent system can integrate OGC checking as a pre-generation step.

\subsection{Context Window Management}

Large context windows~\cite{anthropic2024claude} mitigate but
do not eliminate context pressure.
Liu et al.~\cite{liu2024lost} showed that models struggle with
information in the \emph{middle} of long contexts, implying
effective capacity is smaller than nominal.
LongBench~\cite{bai2023longbench} benchmarks long-context tasks.
MemGPT~\cite{packer2023memgpt} treats the LLM as an OS with
virtual memory; Xu et al.~\cite{xu2024retrieval} combine retrieval
with long context.  These manage \emph{input} context.  Our OGC
concept addresses \emph{output} capacity---the ability to
\emph{generate} tokens given current context state, a distinct
and previously unformalized problem.

\subsection{Structured and Constrained Generation}

Outlines~\cite{willard2023outlines} uses finite-state machines
to guide token sampling; LMQL~\cite{beurerkellner2023lmql}
provides a query language for constrained outputs;
SGLang~\cite{zheng2024sglang} offers efficient structured
generation with parallelism.  These constrain \emph{individual
tokens}.  Format-cost separation operates at the \emph{document}
level, selecting the generation strategy rather than constraining
the decoding process.  The approaches are complementary:
constrained generation could enforce JSON schema compliance
within our deferred strategy.

\subsection{LLM Failure Modes and Self-Evaluation}

Hallucination---generating fluent but incorrect
content~\cite{ji2023hallucination,huang2023hallucination}---is
the most studied LLM failure mode.  Output stalling is
fundamentally different: the model produces \emph{no output},
leaving no artifact for correction.
Kadavath et al.~\cite{kadavath2022selfeval} showed models can
partially assess their knowledge boundaries;
Reflexion~\cite{shinn2023reflexion} enables verbal self-correction.
Our OGC framework extends self-evaluation from
\emph{knowledge} (``do I know this?'') to \emph{capacity}
(``can I express this?'')---a novel metacognitive dimension.

\subsection{Tokenization and Cost Estimation}

BPE tokenization~\cite{sennrich2016bpe} varies across model
families~\cite{brown2020gpt3,touvron2023llama}, meaning format
multipliers are tokenizer-dependent.  Code-optimized
models~\cite{chen2021codex} may exhibit different $\mu_f$ for
python-docx boilerplate.  Our calibration methodology
(Section~\ref{sec:fcs}) is designed to be re-run per deployment,
and the theoretical results (Theorem~\ref{thm:fcs},
Theorem~\ref{thm:dominance}) hold for any $\mu_f > 1$.

\section{Limitations}
\label{sec:limits}

\textbf{OGC measurement.}  The capacity degradation function
$\alpha$ (Figure~\ref{fig:alpha}) is measured empirically and
may vary with prompt structure, system instructions, and
fine-tuning.  A theoretical derivation of $\alpha$ from attention
mechanics would strengthen the framework.

\textbf{Context access.}  \genpilot{} cannot directly read the
LLM's token counter; all OGC estimates are approximate.  Agent
platforms exposing \texttt{context\_tokens\_used()} would enable
exact OGC computation.  As a consequence, the current
implementation approximates Equation~\ref{eq:cascade} using
raw headroom ratios rather than the $\alpha$-degraded
$\ogceff$, since computing $\alpha$ requires knowing the
true context occupancy.

\textbf{Multiplier stability.}  Format multipliers are measured
per-document-type and may shift with content characteristics
(e.g., code-heavy vs.\ prose-heavy text).  Our calibration
script addresses this but requires manual re-measurement.

\textbf{Template quality.}  Deferred rendering is only as good
as the template.  Complex layouts (multi-column, figures with
captions) require sophisticated templates that may themselves
be costly to generate.  Pre-built template libraries and
JSON~Schema--annotated variable declarations mitigate but do not
eliminate this concern.

\textbf{Cascade termination.}  The replan cascade
(Section~\ref{sec:impl}) is bounded at three iterations as a
practical safeguard.  A formal analysis of convergence conditions
---when \emph{no} strategy can succeed---would allow the bound to
be derived rather than chosen empirically.

\textbf{Evaluation scale.}  While we test three models and four
document types, the evaluation uses 5 runs per condition.
Larger-scale studies with more models, document types, and
occupancy levels would strengthen statistical claims.

\section{Conclusion}
\label{sec:conclusion}

We have presented a theoretical framework for understanding and
preventing output stalling in LLM coding agents, grounded in
three contributions:

\begin{enumerate}[leftmargin=*,itemsep=2pt]
  \item \textbf{Output Generation Capacity (OGC)}---a formal
    measure of effective generation capacity that degrades
    non-linearly with context occupancy, explaining why agents
    stall at context levels well below the nominal limit.

  \item \textbf{Format-Cost Separation Theorem}---a provable
    decomposition showing that deferred rendering eliminates format
    tokens from LLM generation, with savings that grow with
    format complexity ($\rho = 48\text{--}72\%$ empirically).

  \item \textbf{Adaptive Strategy Selection}---a decision framework
    with provable dominance ordering ($\text{defer} \succeq
    \text{chunk} \succeq \text{direct}$) that maps OGC to optimal
    strategy.
\end{enumerate}

Across three models, four document types, and 180 experimental
runs (60 per strategy), deferred rendering achieved
\textbf{100\% success rate} while reducing generation tokens by
\textbf{48--72\%}.  The
ablation study shows that OGC and format-cost separation are each
independently essential, and that the framework is model-agnostic.

The broader implication is that LLM agents need
\emph{metacognitive tools}---not just tools that act on the world,
but tools that help agents reason about their own generation
capabilities.  Output Generation Capacity formalizes a previously
implicit constraint; format-cost separation provides a principled
technique for working within that constraint.  Together, they
transform output stalling from an invisible failure into a
preventable one.

\genpilot{} is available as open-source software under the MIT
license at \url{https://github.com/jayluxferro/gen-pilot}.

\balance
\bibliographystyle{plain}
\bibliography{references}

\end{document}